\RequirePackage{fix-cm}
\documentclass[smallextended]{svjour3}       
\smartqed  
\usepackage{graphicx}
\usepackage{microtype}
\usepackage{graphicx}
\usepackage{subfigure}
\usepackage{booktabs} 

\usepackage{hyperref}

\usepackage{amsmath}
\usepackage{algorithm}
\usepackage{algorithmic}
\usepackage[algo2e,vlined,ruled]{algorithm2e}
\usepackage[round]{natbib}

\usepackage{amsfonts}

\usepackage{tikz}
\usetikzlibrary{shapes,arrows}

\DeclareMathOperator*{\argmin}{arg\,min}

 \journalname{Computational Optimization and Applications}
\begin{document}
\title{Make $\ell_1$ Regularization Effective \\in Training Sparse CNN
\thanks{This work was partially supported by the Penn State and Peking University Joint Center for Computational Mathematics and Applications, the Beijing International Center for Mathematical Research from Peking University, and the Verne M. William Professorship Fund from Penn State University.  The research of L. Zhao and L. Zhang was also supported by the China Scholarship Council (for visiting Penn State) and by HKUST16301218 Hong Kong RGC Competitive Earmarked Research Grant (for visiting Penn State), respectively. The authors wish to thank Drs. Lin Xiao and Liang Yang for helpful suggestions and discussions.}
}



\author{
	Juncai He$^1$
	\and
	Xiaodong Jia$^2$
	\and 
	Jinchao~Xu$^1$
	\and
	Lian Zhang$^1$
	\and
	Liang Zhao$^3$
}



\institute{ 
	Jinchao Xu, corresponding author \at
	Tel.: +1 814-865-1110;  \email{jxx1@psu.edu}
	\and
	$^1$ Department of Mathematics, Pennsylvania State University, University Park, PA 16802, USA
	\\
	$^2$ Department of Computer Science and Engineering, Pennsylvania State University, University Park, PA 16802, USA
	\\
	$^3$
	{State Key Laboratory of Scientific and Engineering Computing, Academy of Mathematics and Systems Science, Chinese Academy of Sciences, and University of Chinese Academy of Sciences, Beijing, 100190, China}
}

\date{Received: 24 Aug 2019  / Accepted: 30 May 2020}

\maketitle

\begin{abstract}
Compressed Sensing using $\ell_1$ regularization is among the most powerful and popular sparsification technique in many applications, but why has it not been used to obtain sparse deep learning model such as convolutional neural network (CNN)?  This paper is aimed to provide an answer to this question and to show how to make it work.  {Following~\cite{xiao2010dual}}, We first demonstrate that the commonly used stochastic gradient decent (SGD) and variants training algorithm is not an appropriate match with $\ell_1$ regularization and then replace it with a different training algorithm based on a regularized dual averaging (RDA) method.  {The RDA method of \cite{xiao2010dual}} was originally designed specifically for convex problem, but with new theoretical insight and algorithmic modifications (using proper initialization and adaptivity), we have made it an effective match with $\ell_1$ regularization to achieve a state-of-the-art sparsity for the highly non-convex CNN compared to other weight pruning methods without compromising accuracy (achieving 95\% sparsity for ResNet-18 on CIFAR-10, for example).

\end{abstract}

\section{Introduction}\label{sec_intro}
This paper is devoted to the training of sparse deep neural networks.  In the many successful applications of deep learning \cite{lecun2015deep}, the number of weights in most of the relevant models is often much more than the number of data available (c.f. \cite{pratt1988comparing,han2015deep,he2016deep}).  It is therefore of great theoretical and practical interests to develop numerical methods to {reduce} such weight redundancy and hence compress the network models.  The aim of this paper is to study sparse training algorithms for a special class of deep neural networks, namely convolutional neural networks (CNN). 

As summarized in {\cite{cheng2017survey}}, {roughly speaking}, there are four major different methods that have been developed for compressing neural network models: (1) network pruning and sharing, (2) low-rank factorization, (3) transferred/compact convolutional filters and, (4) knowledge distillation.  
In particular, the network pruning is the most popular compressing method due to its good compatibility and competitive performance and it is also the one that the current paper focuses on.  

Among the many possible approaches for network pruning, the widely used compressed sensing with $\ell_1$ regularization \cite{donoho2006compressed,candes2006robust} appears to be an obvious choice.  One natural step is to first add a proper multiple of the $\ell_1$ norm of the weights to a standard loss function and then train the resulting model with the most commonly used training algorithms such as SGD \cite{mine1981minimization}.  But this approach, as observed in \cite{han2015learning},  does not give satisfactory sparse results for CNN models.  Another approach \cite{Langford2009Sparse,Bertsekas2011Incremental} is to zero out the weights under a threshold at each iteration by using a proximal SGD (Prox-SGD). As explained in \S\ref{sec_alg}, this approach is slightly more efficient than the above method, but still generates very limited sparsity due to its decaying soft-thresholding parameter.

Perhaps due to the aforementioned non-satisfactory performances of SGD when applied to $\ell_1$ regularization, no reports can be found in the literature on any successful application of compressed sensing technique with $\ell_1$ regularization to deep neural networks.  Such a situation is, however, different in the context of convex optimization such as logistic regressions.  \cite{xiao2009dual} successfully developed a special compressed sensing technique for convex machine learning models.  In this work, he also observed that the SGD type method is not effective when used with $\ell_1$ regularization.  Instead he turns to the simple dual averaging method (SDA) \cite{nesterov2009primal},  which is specifically designed for convex optimization problems.  By combining SDA with  $\ell_1$ regularization,  \cite{xiao2009dual} developed the regularized dual averaging (RDA) method and obtained very satisfactory sparse solutions of convex stochastic regularized problems. 

One natural question is if the idea in \cite{xiao2009dual} can be generalized to deep neural networks that are often highly non-convex.  But, with an extensive literature search, we have not yet found any works that discuss such a generalization.  In fact, we could not find any works that use SDA type of methods for the training of any machine learning models that are not convex.  Given the fact that SDA method is originally designed for convex problems {naturally}, SDA is not expected to work {and has never been applied} for non-convex problems, not to mention non-convex problems together with $\ell_1$ regularization.

Despite of {these historic developments}, we report in this paper that SDA, with {some appropriate} modification,  can also be made highly effective with $\ell_1$ regularization to obtain sparse convolutional neural networks.  Our work is motivated by a critical observation that we made and report in this paper: SDA can be interpreted as a perturbation of SGD!  Since SGD is a good training algorithm for CNN, we expect that SDA is potentially also a good training algorithm for CNN. Furthermore, we {demonstrate} that SDA can be combined with a soft-thresholding operator in the forward-backward splitting form to obtain an RDA algorithm for training sparse CNN.

With careful theoretical analysis and extensive numerical experiments, we find that the effectiveness of our RDA method depends crucially on two important techniques, namely (1) proper initialization, and (2) adaptive sparse retraining. The first one is the key for RDA to work with CNN, and the second one {further} improves both sparsity and accuracy.   Consequently, our RDA training process can be described as a two-step pipeline: (1) train CNN by RDA with a specific initialization, and (2) apply adaptive sparse retraining. These two steps lead to state-of-the-art  performance of  RDA to achieve high sparsity for CNN without compromising accuracy in comparison with other weight pruning methods. 

The remainder of this paper is organized as follows. In \S\ref{sec_alg}, we briefly review SGD, SDA, Prox-SGD and RDA, then provide a comparison of these methods to explain why RDA performs better than the other methods. We describe in   \S\ref{sec_add} two techniques that are essential in utilizing RDA. 
Following the detailed implementation listed in \S\ref{sec_num}, we show the numerical results of different methods and compare them with some existing work. In \S\ref{sec_conc}, we summarize our results.

\section{Related works}
Recently, there {have been many discussions in the literature} on the value of network pruning. \cite{liu2018rethinking} reviews various pruning methods and proposes that the value of network pruning is to search good architectures. \cite{Mittal2018Recovering} shows that a randomly pruned network has comparable performance to the original one due to its plasticity. \cite{Zhu2017To} argues that pruned large sparse models outperform small-dense models, although their memory footprints are almost the same, {and hence} indicates that network pruning is meaningful in practice. 

In general, network pruning includes individual weight pruning and structured pruning. The earliest examples of individual weight pruning methods are Optimal Brain Damage \cite{lecun1990optimal} and Optimal Brain Surgeon \cite{hassibi1993second}. Recently, \cite{han2015learning} presents a general three-step pipeline: training, pruning and fine-turning. Typically, individual weight pruning can only guarantee the sparsity of weight matrices, but does not necessarily lead to compression and speedup without the support of specific hardware and libraries. A three-stage pipeline is proposed by \cite{han2015deep} to reduce the storage and energy required to run the networks. The first stage is based on the individual weight pruning in \cite{han2015learning}, followed by quantization and Huffman coding stages to reduce the storage.

Structured pruning, on the other hand, aims to prune the filters or channels. Filters can be pruned based on their corresponding $\ell_1$ norm \cite{li2016pruning}. Similarly, some other methods prune filters based on the information of output channels \cite{hu2016network, luo2017thinet, he2017channel}. Group sparsity is also widely used in the pruning process after training. \cite{wen2016learning} proposes a group sparsity strategy including filter-wise, channel-wise, shape-wise and depth-wise structured sparsity. \cite{alvarez2016learning} makes use of a group regularizer on the neurons of the fully connected layers. \cite{liu2017learning} utilizes the scaling factors in BN layers as a metric to prune filters. \cite{huang2017data} selects sparse structures by imposing sparsity constraints on the outputs of specific structures, such as neurons, groups or residual blocks.

Compressed sensing with $\ell_1$ regularization has been successfully used in many applications \cite{eldar2012compressed, lustig2007sparse}. As an important technique, $\ell_1$ regularization is also adopted in machine learning fields to obtain sparse model in specific learning problems. In past few years, numerous algorithms are designed to find solutions of regularized convex optimization problems. Among them, the Prox-SGD method, also known as FOBOS \cite{duchi2009efficient} in forward-backward splitting form \cite{Lions1979Splitting}  {has been} used in deep learning, for example in  \cite{huang2017data}. As noticed in \cite{xiao2010dual}, one drawback of Prox-SGD is that the thresholding parameters will decay in the training process, which results in unsatisfactory sparsity. {Thus \cite{xiao2010dual} developed the RDA method} to obtain more sparse solution while keeping the accuracy, {and further established the convergence of his RDA method} for convex problems.  {But Xiao's RDA method} has not yet been applied in deep learning thus far. 

\section{Algorithms using $\ell_1$ regularization}\label{sec_alg}
In this section, we first briefly review SGD and SDA as training algorithms for deep learning, and prove that SDA can be viewed as a perturbation of SGD. {We then} introduce Prox-SGD and RDA for $\ell_1$ regularized problems. In the equivalent forward-backward splitting form, these two algorithms can be viewed as iteratively using SGD or SDA with soft-thresholding. Finally, we explain why RDA is much more effective than Prox-SGD {for obtaining} sparsity, which motivates its use to a sparse training algorithm for deep learning.

\subsection{SGD and SDA}\label{sgd_sda}
Consider a classification problem. Let $z=(x,y)$ be an input-output pair of data, such as a picture and its corresponding label. Let $w$ be weights in the model, and $f(w,z)$ be the loss function corresponding to $z$ and $w$. Our aim is to solve the optimization problem
\begin{equation}\label{approximated optmization}
\min_{w}~~\left\lbrace \frac{1}{n}\sum_{z\in Z} f(w,z)\right\rbrace,
\end{equation}
where $Z=\lbrace z_1, z_2, \dots, z_n \rbrace$ is the dataset.

SGD is a commonly used algorithm for solving \eqref{approximated optmization}. 
The major step in SGD with mini-batch can be represented as: 
\begin{equation}\label{eq:sgd}
w_{t+1} = w_t -\eta_t g_t,
\end{equation}
with $g_t= \frac{1}{m}\sum_{z\in X_t}\nabla_wf(w_t, z)$ on mini-batch $X_t\subset Z$. In another form, $w_{t+1}$ can be interpreted as follows
\begin{equation}\label{sub_sgd}
w_{t+1}=\argmin_w~~\left\lbrace g_t^Tw+\frac{1}{2\eta_t}\|w-w_t\|_2^2\right\rbrace.
\end{equation}
Intuitively, the empirical loss function in \eqref{approximated optmization} is replaced with its first-order approximation, then we have $g_t^Tw$ on $X_t$. And regularization term $\frac{1}{2} \|w-w_t\|_2^2$, which uses $w_t$ as moving proximal center, is added to control the distance between $w_{t+1}$ and $w_{t}$. Generally speaking, for convex functions, $\eta_t$ can be taken as $ \frac{1}{\alpha\sqrt{t}}$ in \cite{nemirovsky1983problem} with hyper-parameter $\alpha$. And in real application of CNN, we use the strategy {as discussed in \S \ref{sec:discussion}}.

SDA can be understood as solving a different subproblem at each time
step with respect to the SGD form in \eqref{sub_sgd}. As shown in
\cite{nesterov2009primal}, SDA is primal-dual {type method} since, 
{it generates} a feasible approximation to the optimum of an
appropriately formulated dual problem. Specifically, the update scheme
of SDA we consider here is
\begin{equation}\label{SDASub}
w_{t+1} = \mathop{\arg \min}_w \left\{  \bar{g}_{t}^T w + \frac{1}{2{\xi_t}}\|w-w_c\|^2_2 \right\}.
\end{equation}
where $\bar{g}_{t} =\frac{1}{t}\sum_{\tau=1}^{t} g_{\tau}$. Unlike
SGD, the original loss function in (\ref{approximated optmization}) is
approximated by $\frac{1}{t}\sum_{\tau=1}^{t} g_{\tau}^Tw$, a linear
function obtained by averaging all previous stochastic gradient
$g_\tau$. This sequence corresponds to the support functions $g_\tau^T
w$ in the dual space. Also, the second term establishes a dynamically
updated scale between the primal and dual spaces. The regularization
term $\frac{1}{2}\|w-w_c\|^2_2$ is strongly convex and uses $w_c$ as
fixed proximal center, which is different from SGD. 
{According to RDA in \cite{xiao2010dual}, $w_c = 0$ if we apply the $\ell_1$ 
regularization term. Thus, we will take $w_c = 0$ in the rest of our paper.}
And $\{\xi_t\}$ is a nonnegative and nondecreasing sequence which determines 
the convergence rate. Here, following the idea in \cite{nesterov2009primal}, $\xi_t$ is chosen to
be $\frac{ \sqrt{t}}{\alpha}$.


Originally, the SDA method was designed for solving convex
optimization problems because it was first inspired by convex
combination of linear functions.  
{Some comparison and connections between SDA and SGD are discussed in \cite{mcmahan2011follow, mcmahan2017survey}.
We also show the underlying relation between SGD and SDA with a concise lemma below.
}

\begin{lemma}\label{lemma:SDASGD}
	The {SDA} method is equivalent to the following
	perturbed SGD method:
	\begin{equation}
	\label{SGD1}
	w_{t+1}
	= (1-{\epsilon_t}) w_t -\gamma_t g_t\\
	\end{equation}
	where
	$$
	{\epsilon_t=\frac{1}{t+\sqrt{t^2-t}}<
	\frac{1}{2t-1}},
	$$
	and $\gamma_t={ \frac{\xi_t}{t} =  \frac{1}{\alpha \sqrt{t}}}$.
\end{lemma}
\begin{proof}
The update scheme of \eqref{SDASub} can be rewritten as
{\begin{equation}
\begin{split}
w_{t+1}&=- \xi_t \bar{g}_t  \quad  (w_c = 0 ~~\text{ in}~~ \eqref{SDASub}) \\
&= - \frac{\xi_t}{t} \sum_{\tau=1}^{t}g_\tau\\
&=-\gamma_t \sum_{\tau=1}^{t} g_\tau.
\end{split}
\end{equation}}


Then SDA can be expanded recursively as
\begin{align*}
w_{t+1} &= - \gamma_t \sum_{\tau=1}^{t} g_\tau\\
&=- \gamma_t \left( \sum_{\tau=1}^{t-1} g_\tau + g_t \right)\\
&= \frac{\gamma_t}{\gamma_{t-1}} (-\gamma_{t-1} \sum_{\tau=1}^{t-1} g_\tau) -\gamma_t g_t\\
&= \frac{\gamma_t}{\gamma_{t-1}} w_t -\gamma_t g_t\\
&= (1-{\epsilon_t}) w_t - \gamma_t g_t,
\end{align*}
where
\begin{equation}
{\epsilon_t=\frac{1}{t+\sqrt{t^2-t}}<
\frac{1}{2t-1}}.
\end{equation} 
This finishes the proof. \qed
\end{proof}

Thus SDA can be viewed as a perturbation of SGD, since,  as either $t$ is sufficiently large
\begin{equation}
\label{eq:1}
{1 - \epsilon_t} = \frac{\gamma_t}{\gamma_{t-1}} =\sqrt{1-\frac{1}{t}}\approx 1,
\end{equation}
and $\gamma_t={ \frac{\xi_t}{t} =  \frac{1}{\alpha \sqrt{t}}}$.
From the lemma above, SDA may potentially have similar efficiency with SGD  in
solving non-convex problems, even applied to deep learning fields.

\begin{lemma}
Let $w_t$ and $\tilde w_t$ be the sequences generated by SGD and SDA
respectively.  Then
$$
w_t-\tilde w_t\to 0
$$
as $t\to \infty$, in some appropriate sense. 
\end{lemma}

\subsection{$\ell_1$ regularization, sparsity and algorithms}
A natural idea to obtain a sparse CNN model is to add an $\ell_1$ regularization term to the loss function, which is a well-known technique in compressed sensing \cite{donoho2006compressed}. In other words, we hope to achieve the sparsity by solving the following regularized problem 
\begin{equation}\label{regularized opt problem}
\min_{w}~~\left\lbrace \phi(w) = \frac{1}{n}\sum_{z\in Z} f(w,z) + \lambda\|w\|_1\right\rbrace,
\end{equation}
where $\lambda$ is a hyper-parameter which controls the sparsity of solution. Despite the fact that there is no rigorous theory to prove the sparsity for the solution of such a complex model \eqref{regularized opt problem}, numerical soft-thresholding introduced by the $\ell_1$ norm may generate sparsity at the cost of accuracy. That is to say, an appropriate training algorithm with $\ell_1$ regularization may achieve sparsity with acceptable accuracy. Naturally, we have the following two strategies for solving the above problem:
\begin{itemize}
	\item Prox-SGD:  add the $\ell_1$ regularization into \eqref{sub_sgd}, which will be discussed in \S \ref{Prox-SGD}.
	\item RDA: add the $\ell_1$ regularization into \eqref{SDASub}, which will be discussed in \S \ref{sec:rda}.
\end{itemize}

Before these two algorithms are introduced, the soft-thresholding operator related to $\ell_1$ regularization defined as entry-wised form
\begin{equation}\label{eq:soft-th}
(\text{soft}(x,\delta))^{(i)}= \text{sgn}(x^{(i)}) \max \left\lbrace |x^{(i)}|-\delta, 0 \right\rbrace,
\end{equation}
where $i$ is the index of element.
Numerically speaking, we can conclude from the definition of the soft-thresholding operator that the larger the parameter $\delta$, the more sparse the solution we will be.

\subsection{Prox-SGD: applying $\ell_1$ directly to SGD}\label{Prox-SGD}
Adding the regularization term $\lambda\|w\|_1$ to subproblem \eqref{sub_sgd} directly gives prox-SGD as:
\begin{equation}\label{sub_prox-sgd}
w_{t+1}=\argmin_w~~\left\lbrace g_t^Tw+\frac{1}{2\eta_t}\|w-w_t\|_2^2 + \lambda\|w\|_1\right\rbrace.
\end{equation}
With some simple induction, Prox-SGD can be written in the forward-backward splitting (FOBOS \cite{duchi2009efficient}) scheme
\begin{equation}
\begin{aligned}
w_{t+\frac{1}{2}}& =w_t - \eta_t g_t, \\
w_{t+1}& =\mathop{\arg \min}_w \left\{ \frac{1}{2{\eta_t}}\|w-w_{t+\frac{1}{2}}\|^2_2 + \lambda \|w\|_1 \right\},
\end{aligned}
\end{equation}
where the forward step is a single step of SGD, and the backward step is equivalent to a soft-thresholding operator working on $w_{t+\frac{1}{2}}$ with parameter $\eta_t \lambda$.
The learning rate $\eta_t= \frac{1}{\alpha \sqrt{t}}$ to obtain reasonable convergence rate in convex problem.

\subsection{RDA by Xiao: applying $\ell_1$ in a different way}\label{sec:rda}
Regularized dual averaging (RDA) is originally designed for convex online learning and stochastic optimization problems \cite{xiao2010dual}. However, RDA can also be understood as SDA with an additional $\ell_1$ regularization. Based on the analysis in \ref{lemma:SDASGD} connecting of SDA and SGD, and the success of SGD in non-convex optimization, we hope that RDA may also work for non-convex problems, especially for CNN models.

Similar to Prox-SGD, RDA is obtained from adding $\lambda \|w\|_1$ to subproblem (\ref{SDASub}), and it also requires $w_c=0$.  The update scheme takes the form
\begin{equation}\label{RDA subproblem}
w_{t+1}=\mathop{\arg\min}_w \left\{\bar g_{t}^Tw+\frac{1}{2{\xi_t}} \|w\|^2_2 + \lambda \|w\|_1\right\}.
\end{equation}
We can clearly see the underlying relation between Prox-SGD and SGD with soft-thresholding from the forward-backward splitting form. The following induction
\begin{equation}
\begin{aligned}
w_{t+1}&=\mathop{\arg\min}_w \left\{ \bar g_t^T w+\frac{1}{2{\xi_t}} \|w\|^2_2+\lambda \|w\|_1 \right\}\\
&=\mathop{\arg\min}_w \left\{ \frac{1}{2{\xi_t}}\|w+ {\xi_t}\bar g_t \|^2_2+  \lambda \|w\|_1\right\},
\end{aligned}
\end{equation}
gives us the forward-backward splitting of RDA,
\begin{equation}
\begin{aligned}\label{RDA in FOBOS}
w_{t+\frac{1}{2}} &=-{\xi_t} \bar g_t, \\
w_{t+1} &= \mathop{\arg \min}_w \left\{ \frac{1}{2{\xi_t}}\|w-w_{t+\frac{1}{2}}\|^2_2 +  \lambda \|w\|_1 \right\},
\end{aligned}
\end{equation}
where ${\xi_t} = \frac{\sqrt{t}}{\alpha} $ to obtain the best convergence rate in the convex case \cite{xiao2010dual}, and $\alpha$ is hyper-parameter. From (\ref{RDA in FOBOS}), one can see that the forward step is actually SDA's single step and the backward step is the soft-thresholding operator working on $w_{t+\frac{1}{2}}$ with the parameter $ \lambda {\xi_t}= {\frac{\lambda \sqrt{t}}{\alpha}}  $ {as presented in \eqref{eq:soft-th}}.

The final algorithms of Prox-SGD and RDA for CNN with $\ell_1$ regularization term can be found in Algorithm \ref{prox-SGD} and Algorithm \ref{RDA}.
\begin{algorithm}[t]
	\caption{Prox-SGD (Directly applying $\ell_1$ to SGD)}
	\label{prox-SGD}
	\begin{algorithmic}
		\STATE {\bfseries Input:} a dataset $Z$ and a loss function $\frac{1}{n}\sum_{z \in Z} f(w,z) + \lambda \|w\|_1$ where $w$ is a vector of the weights.
		\STATE {\bfseries Initialization:} initialize $w_0$ with the standard method.
		\FOR{$t=1$ {\bfseries to} $T$}
		\STATE Select a mini-batch $X_t$ from the dataset. 
		\STATE Compute $g_t= \frac{1}{m}\sum_{z\in X_t}\nabla_wf(w_t, z)$.
		\STATE Update $w_{t+1}$ with Prox-SGD in element-wised form:
		\begin{equation}
		w_{t+1}^{(i)}=
		\begin{cases}
		w_t^{(i)}-\eta_t (g_t^{(i)} +\lambda), & w_t^{(i)}-\eta_t g_t^{(i)} > \eta_t \lambda,\\
		0, & |w_t^{(i)}-\eta_t g_t^{(i)}| \leq \eta_t \lambda,\\
		w_t^{(i)}-\eta_t (g_t^{(i)} -\lambda), & w_t^{(i)}-\eta_t g_t^{(i)} < -\eta_t \lambda,
		\end{cases}
		\end{equation}
		where where $i$ is the index of the elements.
		\ENDFOR
	\end{algorithmic}
\end{algorithm}

\begin{algorithm}[!htbp]
	\caption{RDA (Applying $\ell_1$ in a different way)}
	\label{RDA}
	\begin{algorithmic}
		\STATE {\bfseries Input:} a dataset $Z$ and a loss function $\frac{1}{n}\sum_{z \in Z} f(w,z) + \lambda \|w\|_1$ where $w$ is a vector of the weights.
		\STATE {\bfseries Initialization:}  randomly choose $w_1$ as introduced in \S\ref{sec_ini}, and set $\bar g_0=0$.
		\FOR{$t=1$ {\bfseries to} $T$}
		\STATE Select a mini-batch $X_t$ from the dataset. 
		\STATE Compute $g_t= \frac{1}{m}\sum_{z\in X_t}\nabla_wf(w_t, z)$.
		\STATE Update $$\bar g_{t}=\frac{t-1}{t}\bar{g}_{t-1} + \frac{1}{t} g_{t}.$$
		\STATE Update $w_{t+1}$ with RDA in element-wise form:
		\begin{equation}\label{formula:RDA}
		w_{t+1}^{(i)}=
		\begin{cases}
		- {\xi_t}(\bar g_t^{(i)} +\lambda), & \bar g_t^{(i)} <-\lambda,\\
		0, & |\bar g_t^{(i)}| \leq   \lambda,\\
		-{\xi_t} (\bar g_t^{(i)} -\lambda), &  \bar g_t^{(i)} > \lambda.
		\end{cases}
		\end{equation}
		where $i$ is the index of the elements.
		\ENDFOR
	\end{algorithmic}
\end{algorithm}

\subsection{Comparison of Prox-SGD and RDA}\label{sec:Comparsion}
The soft-thresholding of  Prox-SGD and RDA are quite different. 
{
\begin{itemize}
	\item In Algorithm \ref{RDA}, we have    
	\begin{equation}\label{eq:RDA_w}
	w_{t+1}^{(i)} = 0,~~ \text{if}~~ |\bar g_t^{(i)}| \leq  \lambda,
	\end{equation}
	where the criterion to zero out $w_{t+1}^{(i)}$ only depends on a constant $\lambda$. 
	\item In Algorithm \ref{prox-SGD}, we have
	\begin{equation}\label{eq:SGD_w}
		w_{t+1}^{(i)} = 0,~~ \text{if}~~ |w_t^{(i)}-\eta_t g_t^{(i)}| \leq \eta_t \lambda,
	\end{equation}
	where $\eta_t =  \frac{1}{\alpha\sqrt{t}}$, thus the criterion in this case depends on $ \frac{1}{\alpha\sqrt{t}} \lambda$, which approaches to $0$ as $t$ goes to infinity. 
\end{itemize}
Considering that $w_t^{(i)}-\eta_t g_t^{(i)}$ will converges to certain point which many not be zero in \eqref{eq:SGD_w}, 
we cannot expect significant sparsity in Algorithm~\ref{prox-SGD} since $\eta_t \lambda$ will approach to $0$. 
However, the right hand term (thresholding value) in \eqref{eq:RDA_w} will keep constant as in RDA, which may produce a better sparsity.
Similar discussions can also be found in \cite{xiao2010dual}.
}

Furthermore, from the formulation of the regularized problem \eqref{regularized opt problem}, one can see that there is a trade-off between the accuracy and the regularization term, which can be concluded as too large regularization term controlled by $\lambda$ can weaken the effect of the loss function. In other words, increasing regularization term $\lambda$ will decrease the accuracy of the model.
Thus, it is necessary to make use of an algorithm which can produce a sparse solution with small $\lambda$. As our analysis above shows, RDA has a good balance between sparsity and accuracy.


\section{Two techniques for RDA in CNN}\label{sec_add}
In this section, we introduce two techniques, the initialization and the adaptive sparse retraining method. The first one is essential for RDA to work, and the second one gives much improvement to the results given by RDA. 

\subsection{Initialization}\label{sec_ini}
In the original paper \cite{xiao2010dual}, {the theoretical analysis requires that $w_c = 0$ and $w_1 = \argmin_{w} \|w\|_1 = 0$ as an initialization. 
Such an initialization is also shown to work very well numerically for convex problem studied in \cite{xiao2010dual}.
Let us examine now how this initialization technique would work for a typical CNN model,
such as VGG~\cite{simonyan2014very}, ResNet~\cite{he2016deep} which we will test in this paper.
We note that a typical CNN model can be written as:
\begin{equation}\label{eq:CNN}
f(w;x) = S(W f_{\rm CNN}(\theta; x) + b ),
\end{equation}
where $S(y) = {\rm Softmax}(y) := \left(\frac{e^{y_i}}{\sum e^{y_i}}\right)$ and $f_{\rm CNN}(\theta; x)$ stands
for the main CNN structure except for the fully connected layer with Softmax. 
Let $w = \{W, b, \theta \}$, where $\{W, b\}$ are the parameters for last fully connected layer with Softmax and
$\theta$ represents all parameters in the main structure of CNN models. 
One simple but important observation is that
all those CNN models satisfying the following property
\begin{equation}\label{eq:CNN=0}
f_{\rm CNN}(0; x)  = 0,
\end{equation}
as long as the underly activation function satisfies 
\begin{equation}\label{eq:sigma}
\sigma(0) = 0.
\end{equation}
This property is satisfied by $\sigma(x)  = {\rm ReLU}(x) := \max\{0,x \}$.
That is to say, if all parameters are chosen as zero, the output of the main structure of a general CNN model will always equal to zero.
Thus, for a general CNN model $f(w;x)$ as in \eqref{eq:CNN} with property~\eqref{eq:CNN=0}, we will have
\begin{equation}
\begin{aligned}\label{key}
\left. \frac{\partial f^{(i)}(w; x)}{\partial W^{(j,k)}} \right |_{w = 0} &= \left.\frac{\partial}{\partial W^{(j,k)}}\left[ {S}( \sum_{p} W^{(i,p)} f^{(p)}_{\rm CNN}(0;x) + b^{(i)})  \right]\right|_{w=0}\\
&= {S}'(b^{(i)}) \delta_{ij} f^{(k)}_{\rm CNN}(0;x) = 0, \quad \forall i, j, k.
\end{aligned}
\end{equation}
That is to say
\begin{equation}\label{eq:fW=0}
\left. \frac{\partial f(w; x)}{\partial W} \right|_{w = 0} = 0.
\end{equation}
Furthermore, 
\begin{equation}
\begin{aligned}\label{key}
\left. \frac{\partial f^{(i)}(w; x)}{\partial \theta^{(j)}} \right |_{w = 0} &= \left.\frac{\partial}{\partial W^{(j,k)}}\left[ {S}( \sum_{p} W^{(i,p)} f^{(p)}_{\rm CNN}(\theta;x) + b^{(i)})  \right] \right|_{w=0}\\
&= {S}'(b^{(i)}) \sum_{p} W^{(i,k)} \left. \frac{\partial f^{(k)}_{\rm CNN}(\theta;x)}{\partial \theta^j} \right|_{w=0}= 0,
\end{aligned}
\end{equation}
for all $i,j$ as $W=0$.
That indicates that
\begin{equation}\label{eq:ftheta=0}
\left. \frac{\partial f(w; x)}{\partial \theta} \right|_{w = 0} = 0
\end{equation}
Considering the observations \eqref{eq:fW=0} and \eqref{eq:ftheta=0} for zero initialization in CNN, we have the next proposition.
\begin{proposition}
	The RDA method with $w_1 = 0$ cannot converge for CNN with activation function $\sigma$ satisfying $\sigma(0)=0$.
\end{proposition}
}

As a result, non-zero initialization is a necessary condition in {all gradient-based training algorithm including} RDA for CNN. 
{Thus we propose to initialize $w_1$ via some random strategies as discussed later in this subsection.}
Actually, this modification will not influence the convergence of the algorithm. As proven in Theorem \ref{Convergence_Modified_RDA}, the convergence rate for convex problems based on this modification is still $\mathcal O(\frac{1}{\sqrt t})$ when ${\xi_t}=\mathcal O(\sqrt{t})$.  
\begin{theorem}\label{Convergence_Modified_RDA}
	Assume the loss function $f(w,z)$ in the problem (\ref{regularized opt problem}) is convex and there exists an optimal solution $w^{\star}$ to the problem (\ref{regularized opt problem}) with $\Psi(w)=\lambda\|w\|_1$ that satisfies $\frac{1}{2}\|w^{\star}\|_{2}^2\leq D^2$ for some $D>0$. {In addition, we assume that we have the next bound for the randomly chosen $w_1$:
	\begin{equation}\label{eq:conditionQ}
	\Psi(w_1)=\lambda \|w_1\|_1 \leq Q.
	\end{equation}}
	 Let the sequences $\{w_t\}_{t\geq 1}$ be generated by Algorithm \ref{RDA}, and assume {$\|g_t\|_{2}\leq G$} for some constant $G$. Then the expected cost $\mathbf{E}\phi(\bar{w}_t)$ converges to $\phi^{\star}$ with rate $\mathcal O(\frac{1}{\sqrt{t}})$
	\begin{equation}\label{eq:convergence}
	\mathbf{E}\phi(\bar{w}_t)-\phi^{\star}= \mathcal O(\frac{1}{\sqrt{t}}),
	\end{equation}
\end{theorem}
with $\bar w_t=\frac{1}{t} \sum_{\tau=1}^{t} w_\tau$ and $\phi^{\star}=\phi(w^{\star})$. \\
%

{
By adding the extra assumption that for the bound of $w_1$ as in \eqref{eq:conditionQ}, 
we can then prove the above result by following Xiao's work in \cite{xiao2010dual} with difference of an extra coefficient in $\mathcal O(\frac{1}{\sqrt{t}})$ which
is related to $Q$. 
The only difference between the original RDA and RDA used in Algorithm \ref{RDA} is that the former one takes $w_1=\argmin_w \|w\|_1 = 0$ as initialization whereas the latter one allows us to chooses $w_1$ randomly. 
This is a small modification in algorithm and proof but it plays a crucial role in applying RDA to CNN as discussed in the beginning of this section.
}

In particular, when the activation function is ReLU, the weights in CNN are usually initialized with a uniform or a normal distribution \cite{lecun2012efficient,glorot2010understanding,he2015delving}. For RDA, we propose to initialize the weights with a uniform distribution $\mathcal{U}(-b,b)$, where 
\begin{equation}
b=\sqrt{\frac{s}{n}}.
\end{equation}
For a convolutional layer, $n=k^2 c$ is the size of the filter, where $c$ is the number of input channels and $k$ is the width of the filter. For a fully connected layer, $n$ is the dimension of the input vector. In both cases, $s$ is a scalar to increase the weights (e.g. \cite{he2015delving} proposes to choose $s=6$). 

\begin{table}[tb]
	\caption{Different initialization scalars on ResNet-18, CIFAR-10 with RDA. This table shows TOP-1 and TOP-5 accuracy on validation dataset. All models are trained for 120 epochs. {({\bf TOP-1 accuracy} is the conventional accuracy, which means that the model answer (the one with the highest probability) must be exactly the expected answer. {\bf TOP-5 accuracy} means that any of your model that gives 5 highest probability answers that must match the expected answer.)}}
	\label{table:ini_CIFAR10}
	\vskip 0.15in
	\begin{center}
		\begin{small}
				\begin{tabular}{lccccr}
					\toprule
					$\sqrt{s}$ &  TOP-1 &  TOP-5 &  Sparsity\\
					\midrule
					1, 2    &  10.00 &  50.00 &  N/A      \\
					3      &  85.52 &  99.24 &  0.98    \\
					4      &  86.72 &  99.45 &  0.97     \\
					5      &  90.03 &  99.44 &  0.95     \\
					10     &  90.67 &  89.50 &  0.94    \\
					100    &  \bf 91.41 &  99.58 &  0.84 \\
					1000   &  90.36 &  99.62 &  0.63 \\
					10000  &  71.80 &  97.94 &  0.34 \\
					20000  &  68.06 &  97.39 &  0.99 \\
					\bottomrule
				\end{tabular}
		\end{small}
	\end{center}		
\end{table}

\begin{table}[tb]
	\caption{Different initialization scalars on ResNet-18, CIFAR-100 with RDA. 
		This table shows TOP-1 and TOP-5 accuracy on validation dataset. All models are trained for 120 epochs.}
	\label{table:ini_CIFAR100}
	\vskip 0.15in
	\begin{center}
		\begin{small}
				\begin{tabular}{lccccr}
					\toprule
					$\sqrt{s}$ &  TOP-1 & TOP-5  &  Sparsity \\ 
					\midrule
					1      &  63.67 &  87.85 &  0.91 \\
					2      &  \bf66.90 &  88.53 & 0.60  \\
					5      &  65.47 &  88.09 &  0.60 \\
					10     &  65.54 &  88.21 &  0.42 \\
					15     &  64.22 &  87.53 &  0.43 \\
					25     &  63.06 &  88.10 &  0.50 \\
					30     &  62.75 &  86.80 &  0.42 \\
					50     &  64.48 &  87.14 &  0.38   \\
					100    &  60.00 &  86.14 &  0.36  \\
					\bottomrule
				\end{tabular}
		\end{small}
	\end{center}
\end{table}

Since $f$ is non-linear, the effect of initialization on $g_1$, the gradient of $w_1$, is not that clear. Assuming that $f$ is a linear function, then $g_1$ is scaled in the same way as $w_1$. Since with a thresholding (ignoring the initial learning rate $\eta_1=1$), $g_1$ becomes the value of $w_2$, the initial value should not be too small, nor should it be too large because of the exploding gradient problem \cite{pascanu2012understanding}, as shown in Table \ref{table:ini_CIFAR10} and Table \ref{table:ini_CIFAR100}. {Here we have the next definition for sparsity of CNN models:
\begin{equation*}
\text{Sparsity}=\frac{\text{the number of zero weights}}{\text{the number of all weights}},
\end{equation*}
for all tables referred later. }

{Finally}, we listed some good choices for $s$ in Table \ref{table:goods}.

\begin{table}[tb]
	\caption{Suitable $\sqrt{s}$ for different models and datasets. ImageNet represents ILSVRC2012.
	}
	\label{table:goods}
	\vskip 0.15in
	\begin{center}
		\begin{small}
				\begin{tabular}{lcr}
					\toprule
					Dataset & Model & $\sqrt{s}$ \\
					\midrule
					CIFAR-10 & ResNet-18 & 10  \\
					& VGG-16bn & 20 \\
					& VGG-19bn & 10 \\
					CIFAR-100& ResNet-18 & 2  \\
					& VGG-16bn & 60 \\
					& VGG-19bn & 40 \\
					ImageNet & ResNet-18 & 2 \\
					\bottomrule
				\end{tabular}
		\end{small}
	\end{center}		
\end{table}

\subsection{Adaptive sparse retraining (ASR)}\label{sec:ASR}
Fine-tuning is a widely used technique that retrains a pruned model, since the pruning method often decreases the accuracy. This is equivalent to fix the weights to be pruned to zero in the original model, and only update the remaining weights.

During our retraining step, we fix the zero weights and update the remaining weights. If there are newly trained zero weights, they will also be fixed. Thus, in retraining, once a weight becomes zero, it will never be updated. This can be viewed as a stronger fine-tuning, and we call this method adaptive sparse retraining, where the optimization method we use is the same as that used in the first phase. This technique helps improve both the accuracy and sparsity of a model, as shown in Table \ref{table:iteretrain}.
\begin{table}[!htb]
	\caption{ASR helps improve both the sparsity and the accuracy. This table shows TOP-1 accuracy on validation dataset, and the sparsity of weights. The dataset is CIFAR-10.}
	\label{table:iteretrain}
	\vskip 0.15in
	\begin{center}
		\begin{small}
				\begin{tabular}{lcccr}
					\toprule
					& \multicolumn{2}{c}{RDA} & \multicolumn{2}{c}{RDA (ASR)} \\
					Model & TOP-1 & Sparsity & TOP-1 & Sparsity\\
					\midrule
					ResNet-18 & 91.34 & 0.87 & 93.47 & 0.95 \\
					VGG-16bn & 93.07 & 0.92 & 93.24 & 0.94 \\
					VGG-19bn & 92.65 & 0.74 & 93.02 & 0.90 \\
					\bottomrule
				\end{tabular}
		\end{small}
	\end{center}		
\end{table}

\section{Experiments}\label{sec_num}
In this section, we compare the results of RDA and other methods. All results of RDA are based on the two techniques introduced in \S\ref{sec_add}. All accuracies are of the validation dataset. 
The implementation is listed as follows. 

All experiments are carried out with PyTorch (pytorch.org){on TITAN V GPU}.
{For Prox-SGD, we use the same strategy with SGD for initialization and take learning rate as in Algorithm~\ref{prox-SGD}. The total epoch number for Prox-SGD as reported in Table~\ref{table:compare_RDA_p} is 120. This number of epochs is reasonable because, first, the accuracy reaches the highest point in the end, and second, due to the decreasing threshold of Prox-SGD, there should not be too many training epochs, otherwise there will be no sparsity in the end as discussed in \S~\ref{sec:Comparsion}.}

For RDA, filters, and weights as well as bias in fully connected layers, are initialized with uniform distribution introduced in \S\ref{sec_ini}. Weights in batch normalization are initialized with default settings in PyTorch (the mean is set to a 0-vector, and the variant is set to a 1-vector).
In all experiments, the training mini-batch size is 128 \footnote{In the original paper \cite{xiao2010dual}, RDA is proposed as an online learning algorithm, which takes one input at each time.}. Models are all first trained by RDA for 2400 epochs, and then RDA with ASR for 1200 epochs. {Furthermore, in Section \ref{sec:discussion}, we have reduced the number of epochs to 300 on CIFAR-10 and CIFAR-100 with ResNet-18 by tuning the parameter $\alpha$ and $\lambda$ }.

ResNet-18 is based on \cite{he2016deep}. VGG-16bn and VGG-19bn are based on \cite{simonyan2014very}, and both are implemented with batch normalization.

\subsection{Numerical results}
\begin{figure}[t]
	\vskip 0.2in
	\begin{center}
		\centerline{\includegraphics[width=\columnwidth]{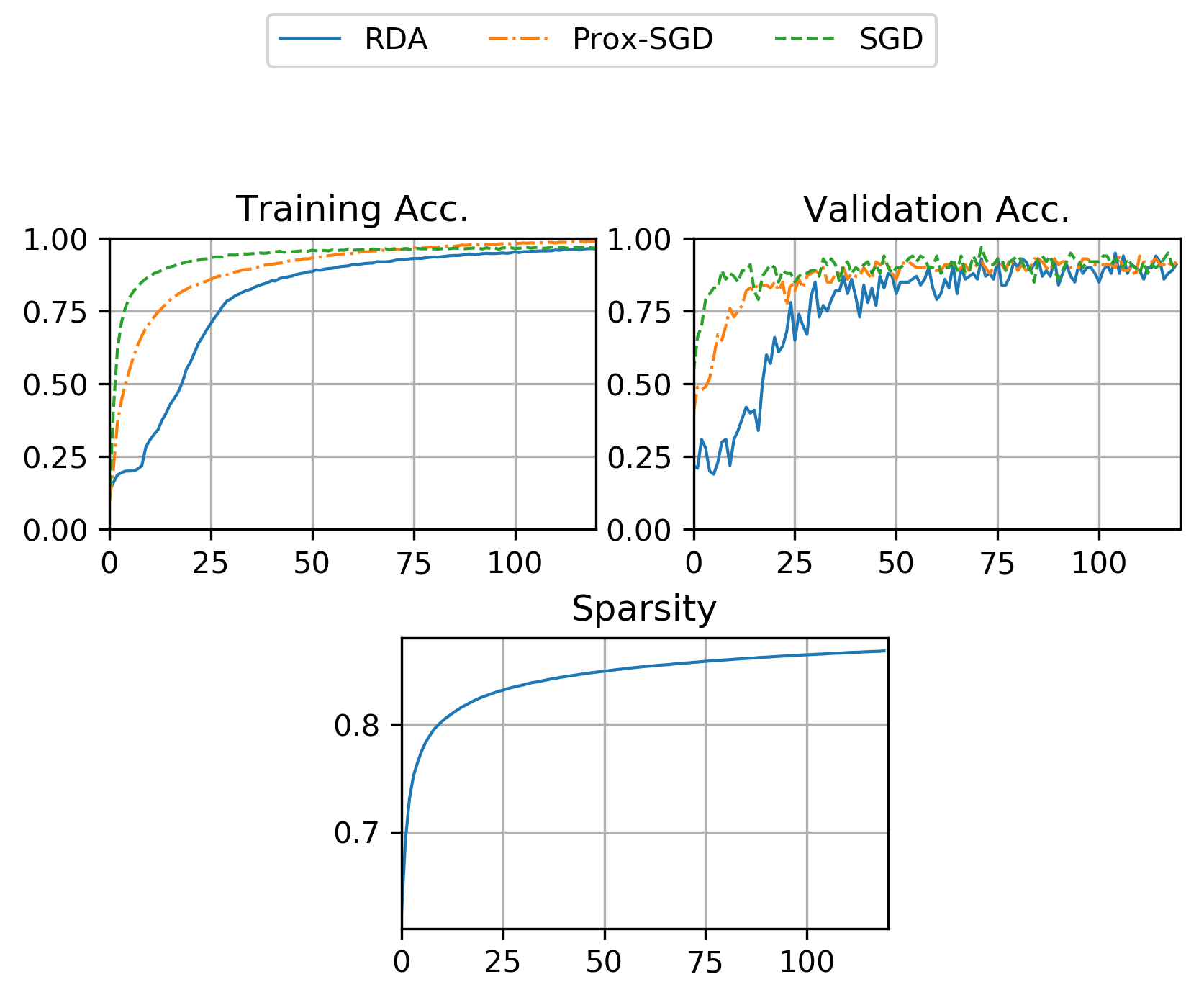}}
		\caption{An example of the first 120 epochs of loss and accuracy curves for different methods, and the sparsity curve of RDA on ResNet-18, CIFAR-10.}
		\label{fig:loss_sparsity}
	\end{center}
	\vskip -0.2in
\end{figure}

\begin{table}[!htb]
	\caption{Compare RDA and prox-SGD for ResNet-18 on CIFAR-10 {with 120 epochs}. RDA achieves better accuracy and sparsity.}
	\label{table:compare_RDA_p}
	\vskip 0.15in
	\begin{center}
		\begin{small}
				\begin{tabular}{lccccr}
					\toprule
					Method &  TOP-1 & TOP-5 & $\lambda$ & $\alpha$ & Sparsity \\
					\midrule
					prox-SGD  	  & 89.80 & 99.40 & $10^{-5}$ & 0.8 & $0.03$  \\ 
					RDA    	  & \bf {91.41} &	99.69 & $10^{-6}$ & 1.0 &	\bf {0.84}    \\ 				
					\bottomrule
				\end{tabular}
		\end{small}
	\end{center}
\end{table}
We first compare RDA and prox-SGD for ResNet-18 on CIFAR-10 {with both 120 epochs} as shown in Table \ref{table:compare_RDA_p}. One can see that RDA performs much better than prox-SGD, and achieves a sparsity of $95\%$. We have analyzed why RDA could be better than prox-SGD in \S \ref{sec_alg}, and the experiments support our claim.

\begin{table}[!htb]
	\caption{RDA on different models and CIFAR-10. RDA works well on different CNN models.}
	\label{table:DiffModel}
	\vskip 0.15in
	\begin{center}
		\begin{small}
				\begin{tabular}{l ccccr}
					\toprule
					MODEL  &   TOP-1 & TOP-5 & $\lambda$ & $\alpha$ & Sparsity    \\ 
					\midrule
					ResNet-18  & 93.47  &	99.69   & $10^{-6}$ & 1.0 & $0.95$  \\
					VGG-16bn  	  & 93.24  & 99.52    & $10^{-6}$ & 1.0 & $0.94$\\
					VGG-19bn     & 93.02  &	99.34     & $10^{-5}$ & 1.0 & $0.90$ \\
					\bottomrule
				\end{tabular}
		\end{small}
	\end{center}
\end{table}
For RDA itself, we show the results on ResNet-18, VGG-16bn and VGG-19bn, CIFAR-10 in Table \ref{table:DiffModel}. One can see that RDA performs well on all models tested. In \S\ref{sec_alg}, we have shown that SDA is a perturbation of SGD, and based on SDA, RDA keeps its general optimization ability on different models.

\begin{table}[!htb]
	\caption{RDA on ResNet-18 and different datasets. RDA works well on different datasets.}
	\label{table:DiffDataset}
	\vskip 0.15in
	\begin{center}
		\begin{small}
				\begin{tabular}{l ccccr}
					\toprule
					Dataset &   TOP-1 &   TOP-5  & $\lambda$ & $\alpha$ & Sparsity \\ 
					\midrule
					MNIST   & 99.63 & 100.00  & $10^{-6}$ & $0.1$ & 0.95   \\
					CIFAR-10 & 93.47  & 99.69  &	 $10^{-6}$ & 1.0  &0.95  \\
					CIFAR-100 & 72.29  & 89.94  & $10^{-8}$ & 0.09  & 0.56  \\
					ImageNet & 64.93  & 84.92  & $10^{-8}$ & $0.1$ & 0.36\\ 
					\bottomrule
				\end{tabular}
		\end{small}
	\end{center}
\end{table}
Table \ref{table:DiffDataset} shows the results of ResNet-18 on CIFAR-10, CIFAR-100 and ImageNet (ILSVRC2012). In general, RDA performs well on different datasets. For ImageNet, {the typical accuracy of SGD for ResNet should be around 69\%}. In some sense, ResNet-18 could lack the redundancy to be sparse while maintaining satisfactory accuracy. A larger model may help improve the performance.


\begin{table}[tb]
	\caption{To compare RDA with the three-step pipeline in \cite{han2015learning}, {we adapt the implementation in \cite{liu2018rethinking}} where the model is first trained by SGD {with 160 epochs}, then pruned according to the sparsity, and finally fine tuned {with 40 epochs} to retrieve the performance (denoted as Model (Han)).  The results of RDA are comparable to \cite{han2015learning}.}
	\label{table:compare_to_han}
	\vskip 0.15in
	\begin{center}
		\begin{small}
				\begin{tabular}{llcr}
					\toprule
					Dataset & Model & TOP-1 & Sparsity\\
					\midrule
					CIFAR-10 
					& ResNet-18 (RDA) & 93.47 & 0.95 \\
					& ResNet-18 (Han) & 93.95 & 0.95 \\
					& VGG-16bn (RDA) & 93.24 & 0.94 \\ 
					& VGG-16bn (Han) & 93.55 & 0.94 \\ 
					& VGG-19bn (RDA) & 93.02 & 0.90 \\ 
					& VGG-19bn (Han) & 93.60 & 0.90 \\
					CIFAR-100
					& ResNet-18 (RDA)& 72.29 & 0.56 \\
					& ResNet-18 (Han) & 74.67 & 0.56  \\
					& VGG-16bn (RDA) & 69.04 & 0.67 \\ 
					& VGG-16bn (Han) & 73.56 & 0.67 \\
					& VGG-19bn (RDA) & 67.46 & 0.48\\ 
					& VGG-19bn (Han) & 72.52 & 0.48 \\
					\bottomrule
				\end{tabular}
		\end{small}
	\end{center}		
\end{table}

Table \ref{table:compare_to_han} compares RDA with the three-step pipeline in \cite{han2015learning}, based on the source code provided by \cite{liu2018rethinking}.
\cite{han2015learning} proposes the pipeline to compress CNN models, where the first step is training a model, the second is pruning a given percentage of weights in the trained model, and the third is fine tuning it. We compare the two methods based on the same sparsity, i.e. if a model trained by RDA has sparsity 0.95, then the model trained by SGD will be pruned $95\%$ weights and then fine tuned. One can see that the results of RDA are comparable to \cite{han2015learning}. This shows RDA is a powerful sparse optimization method for CNN.
\subsection{Additional heuristic techniques}\label{sec:discussion}
{ The numerical results presented in the previous subsections
	show that RDA works well in CNN. Next, we present some heuristic
	techniques that help improve the performance of RDA.  In training
	algorithms like SGD and RDA, when the iteration step $t$ gets large,
	the learning rate becomes too small to lead to any significant
	update of the weights at each step.  In order to solve this problem,
	we developed some heuristic strategy for parameter turning for RDA.
	For example, we modify the parameter $\alpha$ and $\lambda$ after
	training appropriate number of epochs, which can speed up the training process
	significantly according to our investigation on CIFAR-10 and CIFAR-100
	with ResNet-18 shown in Table \ref{table:AdaptiveRDACIFAR-10} and
	Table \ref{table:AdaptiveRDACIFAR-100}.
	{ By trial and error, $\{10^{-5},10^{-6},10^{-7}\}$ is a suitable search space for the parameter $\lambda$, and for the parameter $\alpha$, it should be decreasing during the training process.}
	These numerical experiments reveal the possibility to speed up and
	improve RDA with suitable adaptive parameter strategies. How to
	automatically find a proper adaptive parameter on different datasets
	and networks by theoretical analysis and more parameters tuning is
	still under further investigation.  }
\begin{table}[!htb]
	{
		\caption{RDA with adaptive $\alpha$ helps speed up the training process. The dataset is CIFAR-10 and the network is ResNet-18.}
		\label{table:AdaptiveRDACIFAR-10}
		\vskip 0.15in
		\begin{center}
			\begin{small}
				\begin{tabular}{lcccr}
					\toprule
					Epochs  & $\alpha$  & $\lambda$ & TOP-1 & Sparsity\\
					\midrule
					$[~~~1,100]$ & 1    &  $10^{-5}$ & 92.19 & 0.9257 \\
					$[101,200]$  & 0.2  &  $10^{-5}$ & 93.38 & 0.9422 \\
					$[201,300]$  & 0.05 &  $10^{-5}$ & 93.13 & 0.9645 \\
					\bottomrule
				\end{tabular}
			\end{small}
		\end{center}	
	}	
\end{table}

\begin{table}[!htb]
	{
		\caption{RDA with adaptive $\alpha$ helps speed up the training process. The dataset is CIFAR-100 and the network is ResNet-18.}
		\label{table:AdaptiveRDACIFAR-100}
		\vskip 0.15in
		\begin{center}
			\begin{small}
				\begin{tabular}{lcccr}
					\toprule
					Epochs  & $\alpha$  & $\lambda$ & TOP-1 & Sparsity\\
					\midrule
					$[~~~1,100]$ & 0.28    &  $10^{-6}$ & 68.2  & 0.7455 \\
					$[101,200]$  & 0.21    &  $10^{-7}$ & 71.25 & 0.6842 \\
					$[201,300]$  & 0.08    &  $10^{-6}$ & 72.67 & 0.7782 \\
					\bottomrule
				\end{tabular}
			\end{small}
		\end{center}
	}		
\end{table}
%
%

\section{Concluding remarks}\label{sec_conc}

In contrary to the common perception that the SDA method of
\cite{nesterov2009primal} should only work for convex optimization
problem for which the SDA was originally designed, in this paper, we
manage to make this method as an effective training algorithm for the
highly non-convex CNN. In particular, by combining it with
$\ell_1$ regularization, we develop the corresponding RDA method that
proves to be very effective to obtain sparse CNN models without
compromising generalization accuracy.  The theoretical foundation of
this approach is based on a critical observation we make, namely the
SDA method (with a slight modification) is equivalent to a small
perturbation of the SGD method if the learning rate is chosen
appropriately. While our work is motivated by \cite{xiao2010dual} for
convex optimization problem, we find that the effectiveness of our RDA
method depend crucially on proper initialization and adaptive sparse
retraining.  Preliminary numerical experiments show that our new
method can be used to train sparse CNN with performances comparable to
the state-of-the-art weight pruning methods \cite{han2015learning}.
We further provide theoretical justification of this method for convex
optimization problems and analysis of the effectiveness of different
choices of hyper-parameters in the algorithm.

\newpage
\bibliographystyle{plainnat}
\bibliography{RDA_COAP}



\end{document}